%% file: main.tex
\documentclass[letterpaper]{article} %
\usepackage{mefomo_2023,times}

\usepackage{amsmath}
\usepackage{amssymb}
\usepackage{amsthm}
\usepackage{graphicx}
\usepackage{import}
\usepackage[frozencache]{minted} %
\usepackage{pgfplots}
\usepackage{subcaption}
\usepackage{textcomp}
\pgfplotsset{compat=1.17}

\usepackage{url}
\usepackage{hyperref}
\usepackage{cleveref}

\title{Cliff-Learning}

\author{Tony T. Wang\textsuperscript{1} \quad Igor Zablotchi\textsuperscript{1}\textsuperscript{2} \quad Nir Shavit\textsuperscript{1}\textsuperscript{3} \quad Jonathan S. Rosenfeld\textsuperscript{1}\\
\vspace{-8pt}\\
\texttt{\{ttw,shanir,jonsr\}@csail.mit.edu} \quad \texttt{igor.zablotchi@gmail.com} \\
\vspace{-8pt}\\
\textsuperscript{1}MIT \quad \textsuperscript{2}Mysten Labs \quad \textsuperscript{3}Neural Magic
\vspace{-10pt}
}

\newif\ifisarxiv
\isarxivtrue

\ifisarxiv
    \iclrfinalcopy
\fi

\makeatletter
\newcommand\thefontsize[1]{{#1 The current font size is: \f@size pt\par}}
\makeatother

\newtheorem{proposition}{Proposition}[section]
\numberwithin{equation}{section}
\numberwithin{figure}{section}

\DeclareMathOperator*{\argmin}{arg\,min}

\begin{document}

\vspace*{0.05in}

\maketitle
\ifisarxiv
    \lhead{Preprint}
\fi

\begin{abstract}
    We study the data-scaling of transfer learning from foundation models in the low-downstream-data regime.
    We observe an intriguing phenomenon which we call \textit{cliff-learning}.
    Cliff-learning refers to regions of data-scaling laws where performance improves at a faster than power law rate (i.e. regions of concavity on a log-log scaling plot). We conduct an in-depth investigation of foundation-model cliff-learning and study toy models of the phenomenon. We observe that the degree of cliff-learning reflects the degree of compatibility between the priors of a learning algorithm and the task being learned.
\end{abstract}

\section{Introduction}

The use of pre-trained models as a foundation~\citep{bommasani2021opportunities} for learning has resulted in groundbreaking advances in AI. For example, ChatGPT is a finetuned version of GPT-3~\citep{brown2020language}, and Stable Diffusion~\citep{rombach2022high} is built on top of CLIP-type models~\citep{radford2021learning}. In addition to enabling new capabilities, transfer learning on top of foundation models enables deep-learning to function in scenarios where training data is limited.

One way to evaluate the effectiveness of transfer-learning is to look at its data-scaling behavior: how fast does test error decay when we increase the amount of (downstream) training data? \citet{hernandez2021scaling} finds that when the number of training samples $n$ is large (\Cref{app:related}), transfer-learning test error decays as a power law in $n$:
\begin{equation}
    \label{eqn:power-law-defn}
    \underset{
        \text{$\mathsf{D}$ an i.i.d. dataset of size $n$}
    }{
        \mathbb{E}
    }[
        \text{test-error}(
            \text{model trained on $\mathsf{D}$}
        )
    ] = A \times n^{-\alpha} + E.
\end{equation}
This corroborates a large body of work that finds such power laws are very accurate models of data-scaling for neural networks in the large-$n$ regime~\citep{seung1992statistical,hestness2017deep,rosenfeld2019constructive,kaplan2020scaling,hoffmann2022training}.

Yet, little is known about the low-data regime.
We discover that power laws are not an accurate model of transfer learning data-scaling in the low-data regime; instead, performance often improves at a faster than power law rate. On a log-log plot, this shows up as a concave scaling-law (\Cref{app:concavity}). We dub such regions of concavity \textit{cliff-learning} regions. These regions are arguably the most interesting parts of the data-scaling domain, since they have the highest ``return on data''.

In~\Cref{sec:transfer-learning}, we investigate cliff-learning with foundation models. In~\Cref{sec:toy-models}, we develop toy models for cliff-learning. We discuss implications and open questions in~\Cref{sec:conclusion}.

\input{tex/transfer.tex}

\section{Toy models of cliff-learning}
\label{sec:toy-models}
We now study some toy-models of cliff-learning. Our toy-models have a unifying principle: the more the priors of a learning algorithm match the task being learned, the more cliff-learning occurs.

\input{tex/linreg.tex}

\subsection{Cliff-learning in binary Gaussian classification}
In \Cref{sec:gtm}, we study a binary Gaussian classification problem that exhibits cliff-learning. We are able to work out the exact functional form of cliff-learning in this case (\Cref{eqn:gtm-err-approx-best}), and the resulting cliff-learning curves (\Cref{fig:gtm-main}) look very similar to \Cref{fig:linreg-noised}.

\input{tex/harmonics.tex}

\section{Discussion and open questions}
\label{sec:conclusion}

Our toy models in \Cref{sec:toy-models} suggest stronger cliffs arise from better prior-task compatibility. This corroborates our findings in \Cref{sec:transfer-learning} -- we find that the quality of a foundation model (aka the quality of a learned prior) primarily manifests in data-scaling via the initial cliff-learning regime. In the large-$n$ regime, foundation models of different qualities all scale as power laws with similar $\alpha$s c.f. \Cref{eqn:power-law-defn}, but different quality models have different $A$s and $Es$ due to different amounts of cliff-learning.

Finally, we present some open questions:
\begin{enumerate}
    \item
    In \Cref{sec:toy-models}, we conjectured that when two different types of training methods share the same data-scaling asymptotics, they share similar ``mechanisms'' of learning. For example, if a learning algorithm $\mathcal{A}$ has the same asymptotics as a nearest-neighbors method, does this imply $\mathcal{A}$ is ``memorizing'' the training data?  What other ``mechanisms'' are there?

    \item
    We showed that cliff-learning can be achieved by finetuning foundation models. However, these large foundation models generally follow power law scaling in compute, parameters, and data. Is it possible for large foundation models to exhibit strong cliff-learning when trained from scratch? Or is there some sort of conservation principle that says a large amount of data is required somewhere in the learning pipeline?
\end{enumerate}

\subsubsection*{Acknowledgments}
\addcontentsline{toc}{section}{Acknowledgments}
We thank Eric Michaud and Linghao Kong for helpful discussions.

TTW was supported by funding from the Eric and Wendy Schmidt Center at the Broad Institute of MIT and Harvard.

\section*{References}
\addcontentsline{toc}{section}{References}
\begingroup
    \renewcommand{\section}[2]{}%
    \bibliography{references}
    \bibliographystyle{mefomo_2023}
\endgroup

\newpage

\appendix
\section*{Appendix}
\input{tex/related-work.tex}

\clearpage
\input{tex/appendix-misc.tex}

\clearpage
\input{tex/gaussian-toy-model.tex}

\end{document}

%% file: tex/transfer.tex
\section{Foundation model transfer learning}
\label{sec:transfer-learning}

\Cref{fig:data-pruning-comparison-cifar10} shows a prototypical example of transfer learning cliff: finetuning a LAION-CLIP~\citep{schuhmann2022laion} linear probe on CIFAR-10 produces an initial cliff-learning phase that lasts until the model achieves human parity. This is $1000\times$ more efficient than learning from scratch.\footnote{In \Cref{sec:cliff-failures}, we show that linear probe finetuning does not always result in cliff-learning; sometimes, full model finetuning is needed.}

\Cref{fig:data-pruning-comparison-cifar10} also compares finetuning to data-pruning, which was explored as a method for beating power-law scaling in~\citet{sorscher2022beyond}. We find that data-pruning alone does not perform meaningfully better than training from scratch. Data-pruning combined with transfer learning demonstrates stronger signatures of cliff-learning\footnote{For more commentary on data-pruning, see \Cref{app:related}.}, but is outperformed by a linear probe on top of a stronger foundation model.

In \Cref{fig:transfer-deep-dive}, we compare different methods of finetuning for CLIP models of varying quality. The asymptotic behavior of different finetuning methods can be grouped into a few different categories: linear-probes plateau, full-model finetuning shares the asymptotics of from-scratch training, and deeper fully-connected probes share the same asymptotics as $k$-NN probes. However, these asymptotics are only similar in slope; their intercepts are different due differences in transient cliff-learning dynamics. In other words, the quality of a foundation model primarily influences data-scaling in the regime of cliff-learning, higher quality foundation models cliffing more aggressively, and all models share the same asymptotic behaviors in the large-$n$ regime. We conjecture that shared asymptotics arise due to shared mechanisms of learning. 

\begin{figure}[!b]
    \vspace{-5mm}
    \centering
    \import{res/transfer}{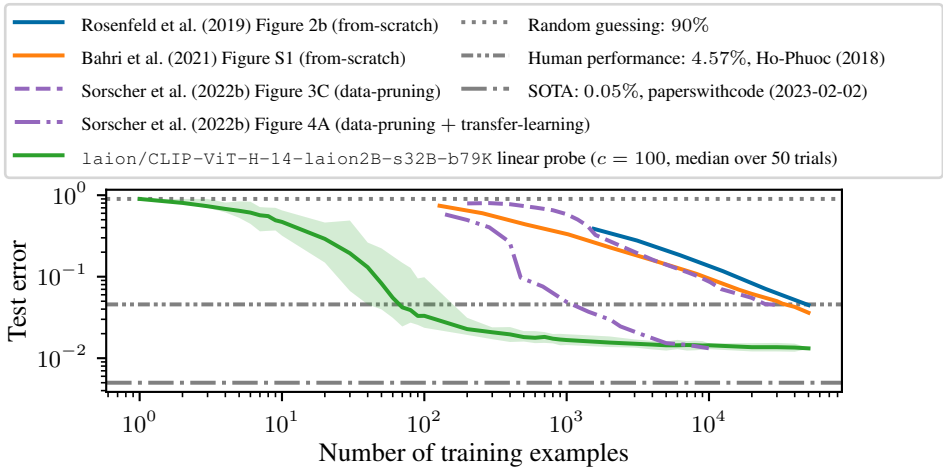}
    \vspace{-3mm}
    \caption{
        Data-scaling of transfer learning, data-pruning, and from-scratch learning on CIFAR-10. The shaded region represents the full range of measured values across all trials. For \citet{sorscher2022beyond}, Pareto-optimal curves across different pruning fractions are plotted.
    }
    \label{fig:data-pruning-comparison-cifar10}
\end{figure}

\begin{figure}[!b]
    \centering
    \import{res/transfer}{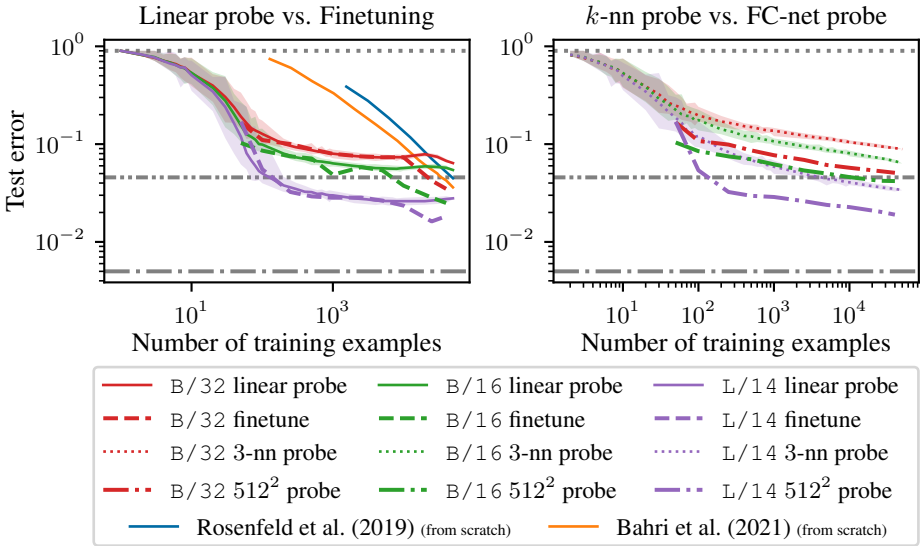}
    \vspace{-3mm}
    \caption{
        Data-scaling of different finetuning methods on CIFAR-10 for three OpenAI CLIP models of increasing quality: $\texttt{B/32} < \texttt{B/16} < \texttt{L/14}$. %
        The method ``finetune'' consisted of finetuning all models parameters after the final readout layer was initialized to the optimal linear probe. The $512^2$ probe was a two layer fully connected ReLU probe of width $512$. Shaded regions are for 50 trials, with medians marked. See \Cref{fig:svhn-deep-dive,fig:all-the-linear-probes} for related results.
    }
    \label{fig:transfer-deep-dive}
\end{figure}

%% file: tex/linreg.tex
\subsection{Cliff-learning in linear regression}
\label{sec:linreg}

Say we want to learn a linear function $f(x) = v^\top x$ in $d$-dimensions (i.e. $v, x \in \mathbb{R}^d$) from training data $\{(x_i, y_i)\}_{i = 1}^n$, where the $y_i = f(x_i) + \mathcal{N}(0, \sigma^2)$ are noisy observations.

When $n < d$, there are infinitely many linear functions that perfectly fit the training data.
In this regime, it is hopeless to achieve small error.
However when $n \geq d$ and $\sigma = 0$, we can perfectly\footnote{
    Also assuming the $x_i$ are linearly independent, which is almost surely true if for example the $x_i$ are sampled i.i.d. from $\mathcal{N}(0, I_d)$.    
} learn $f$ via the following least-squares algorithm:
\begin{equation}
    \label{eqn:linreg-lstsq}
    \hat{f}(x) = \hat{v}^\top x,
    \quad
    \hat{v} = \argmin_{w}\,
    \sum_{i = 1}^n \left(w^\top x_i - y_i\right)^2.
\end{equation}

This least-squares estimator demonstrates cliff-learning -- when the training dataset is at least of size $d$ (the critical threshold), test-error instantly drops to zero. This cliff is shown in Figure~\ref{fig:linreg-noiseless}, where it is contrasted against the power law scaling of a nearest-neighbor estimator. The reason our least-squares estimator demonstrates cliff-learning while a nearest-neighbor estimator does not is because our least-squares estimator forces the learnt estimator to be linear, a prior that matches the task being learned.

When $\sigma > 0$, the least-squares estimator in \Cref{eqn:linreg-lstsq} can suffer from double descent, but adding a regularization term $\lambda \|v\|_2^2$ to its objective recovers a softer-form of cliff-learning (Figure~\ref{fig:linreg-noised}).

\begin{figure}[t]
    \vspace{-4mm}
    \centering
    \begin{subfigure}{0.48\textwidth}
        \centering
        \input{res/linreg/linreg-noiseless.pgf}
    	\caption{Data-scaling curves on a noiseless linear regression task with $d = 5$ and $\|v\|_2 = 1$. The $y$-axis is on a linear scale below $10^{-2}$.}
        \label{fig:linreg-noiseless}
    \end{subfigure}
    \hfill
    \begin{subfigure}{0.48\textwidth}
        \centering
        \input{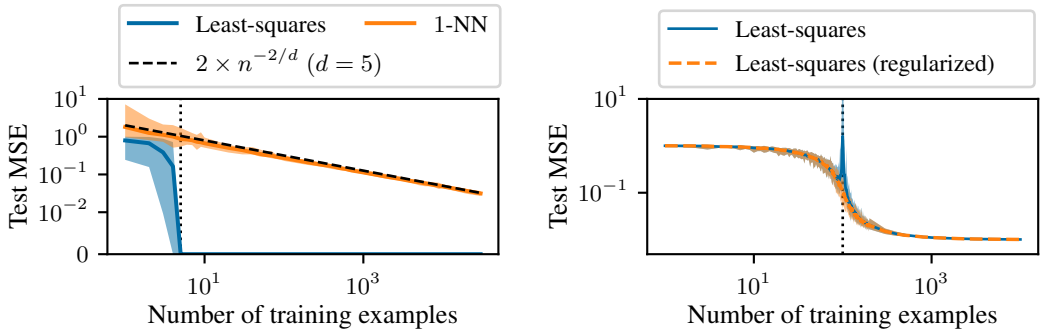}
        \caption{Data-scaling curves on a noisy linear regression task with $d = 100$, $\|v\|_2=1$, and $\sigma = 0.1$. The regularization strength is $\lambda = d \sigma^2 / \|v\|_2^2 = 1$.}
        \label{fig:linreg-noised}
    \end{subfigure}
    \vspace{-1mm}
    \caption{Cliffs arise around the critical threshold $n = d$ (dotted vertical lines) for the linear regression toy problem from \Cref{sec:linreg}. We sample $x_i \overset{\text{i.i.d.}}{\sim} \mathcal{N}(0, I_d)$. 50 random trials are visualized, with shading denoting range and solid lines denoting medians. \textbf{(a)}~On noiseless linear regression, a least-squares estimator exhibits a sharp cliff at $d$ samples, while a nearest neighbor estimator exhibits power law scaling with exponent $-2/d$ (as suggested by theory, c.f. \citet{gyorfi2002distribution} Theorem 6.2). \textbf{(b)} In the noisy setting, regularized least-squares exhibits a soft cliff around $d$ samples, and regularization is needed to avoid double-descent. We set the regularization strength $\lambda$ to the ``optimal'' value c.f. \citet{nakkiran2020optimal}.}
    \label{fig:linreg}
\end{figure}

%% file: res/linreg/linreg-noiseless.pgf
\begingroup%
\makeatletter%
\begin{pgfpicture}%
\pgfpathrectangle{\pgfpointorigin}{\pgfqpoint{2.798903in}{1.893856in}}%
\pgfusepath{use as bounding box, clip}%
\begin{pgfscope}%
\pgfsetbuttcap%
\pgfsetmiterjoin%
\definecolor{currentfill}{rgb}{1.000000,1.000000,1.000000}%
\pgfsetfillcolor{currentfill}%
\pgfsetlinewidth{0.000000pt}%
\definecolor{currentstroke}{rgb}{1.000000,1.000000,1.000000}%
\pgfsetstrokecolor{currentstroke}%
\pgfsetdash{}{0pt}%
\pgfpathmoveto{\pgfqpoint{0.000000in}{-0.000000in}}%
\pgfpathlineto{\pgfqpoint{2.798903in}{-0.000000in}}%
\pgfpathlineto{\pgfqpoint{2.798903in}{1.893856in}}%
\pgfpathlineto{\pgfqpoint{0.000000in}{1.893856in}}%
\pgfpathlineto{\pgfqpoint{0.000000in}{-0.000000in}}%
\pgfpathclose%
\pgfusepath{fill}%
\end{pgfscope}%
\begin{pgfscope}%
\pgfsetbuttcap%
\pgfsetmiterjoin%
\definecolor{currentfill}{rgb}{1.000000,1.000000,1.000000}%
\pgfsetfillcolor{currentfill}%
\pgfsetlinewidth{0.000000pt}%
\definecolor{currentstroke}{rgb}{0.000000,0.000000,0.000000}%
\pgfsetstrokecolor{currentstroke}%
\pgfsetstrokeopacity{0.000000}%
\pgfsetdash{}{0pt}%
\pgfpathmoveto{\pgfqpoint{0.642822in}{0.489984in}}%
\pgfpathlineto{\pgfqpoint{2.689201in}{0.489984in}}%
\pgfpathlineto{\pgfqpoint{2.689201in}{1.303255in}}%
\pgfpathlineto{\pgfqpoint{0.642822in}{1.303255in}}%
\pgfpathlineto{\pgfqpoint{0.642822in}{0.489984in}}%
\pgfpathclose%
\pgfusepath{fill}%
\end{pgfscope}%
\begin{pgfscope}%
\pgfpathrectangle{\pgfqpoint{0.642822in}{0.489984in}}{\pgfqpoint{2.046380in}{0.813271in}}%
\pgfusepath{clip}%
\pgfsetbuttcap%
\pgfsetroundjoin%
\pgfsetlinewidth{1.003750pt}%
\definecolor{currentstroke}{rgb}{0.000000,0.000000,0.000000}%
\pgfsetstrokecolor{currentstroke}%
\pgfsetdash{{1.000000pt}{1.650000pt}}{0.000000pt}%
\pgfpathmoveto{\pgfqpoint{1.026277in}{0.489984in}}%
\pgfpathlineto{\pgfqpoint{1.026277in}{1.303255in}}%
\pgfusepath{stroke}%
\end{pgfscope}%
\begin{pgfscope}%
\pgfpathrectangle{\pgfqpoint{0.642822in}{0.489984in}}{\pgfqpoint{2.046380in}{0.813271in}}%
\pgfusepath{clip}%
\pgfsetbuttcap%
\pgfsetroundjoin%
\definecolor{currentfill}{rgb}{0.000000,0.419608,0.643137}%
\pgfsetfillcolor{currentfill}%
\pgfsetfillopacity{0.500000}%
\pgfsetlinewidth{0.000000pt}%
\definecolor{currentstroke}{rgb}{0.000000,0.000000,0.000000}%
\pgfsetstrokecolor{currentstroke}%
\pgfsetdash{}{0pt}%
\pgfpathmoveto{\pgfqpoint{0.735839in}{1.105401in}}%
\pgfpathlineto{\pgfqpoint{0.735839in}{0.984095in}}%
\pgfpathlineto{\pgfqpoint{0.860924in}{0.946362in}}%
\pgfpathlineto{\pgfqpoint{0.934093in}{0.702442in}}%
\pgfpathlineto{\pgfqpoint{0.986008in}{0.490359in}}%
\pgfpathlineto{\pgfqpoint{1.026277in}{0.489984in}}%
\pgfpathlineto{\pgfqpoint{1.059178in}{0.489984in}}%
\pgfpathlineto{\pgfqpoint{1.086996in}{0.489984in}}%
\pgfpathlineto{\pgfqpoint{1.111093in}{0.489984in}}%
\pgfpathlineto{\pgfqpoint{1.132348in}{0.489984in}}%
\pgfpathlineto{\pgfqpoint{1.151361in}{0.489984in}}%
\pgfpathlineto{\pgfqpoint{1.276446in}{0.489984in}}%
\pgfpathlineto{\pgfqpoint{1.349616in}{0.489984in}}%
\pgfpathlineto{\pgfqpoint{1.401531in}{0.489984in}}%
\pgfpathlineto{\pgfqpoint{1.441799in}{0.489984in}}%
\pgfpathlineto{\pgfqpoint{1.474701in}{0.489984in}}%
\pgfpathlineto{\pgfqpoint{1.502519in}{0.489984in}}%
\pgfpathlineto{\pgfqpoint{1.526616in}{0.489984in}}%
\pgfpathlineto{\pgfqpoint{1.547871in}{0.489984in}}%
\pgfpathlineto{\pgfqpoint{1.566884in}{0.489984in}}%
\pgfpathlineto{\pgfqpoint{1.691969in}{0.489984in}}%
\pgfpathlineto{\pgfqpoint{1.765139in}{0.489984in}}%
\pgfpathlineto{\pgfqpoint{1.817054in}{0.489984in}}%
\pgfpathlineto{\pgfqpoint{1.857322in}{0.489984in}}%
\pgfpathlineto{\pgfqpoint{1.890224in}{0.489984in}}%
\pgfpathlineto{\pgfqpoint{1.918041in}{0.489984in}}%
\pgfpathlineto{\pgfqpoint{1.942138in}{0.489984in}}%
\pgfpathlineto{\pgfqpoint{1.963393in}{0.489984in}}%
\pgfpathlineto{\pgfqpoint{1.982407in}{0.489984in}}%
\pgfpathlineto{\pgfqpoint{2.107491in}{0.489984in}}%
\pgfpathlineto{\pgfqpoint{2.180661in}{0.489984in}}%
\pgfpathlineto{\pgfqpoint{2.232576in}{0.489984in}}%
\pgfpathlineto{\pgfqpoint{2.272845in}{0.489984in}}%
\pgfpathlineto{\pgfqpoint{2.305746in}{0.489984in}}%
\pgfpathlineto{\pgfqpoint{2.333564in}{0.489984in}}%
\pgfpathlineto{\pgfqpoint{2.357661in}{0.489984in}}%
\pgfpathlineto{\pgfqpoint{2.378916in}{0.489984in}}%
\pgfpathlineto{\pgfqpoint{2.397929in}{0.489984in}}%
\pgfpathlineto{\pgfqpoint{2.523014in}{0.489984in}}%
\pgfpathlineto{\pgfqpoint{2.596184in}{0.489984in}}%
\pgfpathlineto{\pgfqpoint{2.596184in}{0.489984in}}%
\pgfpathlineto{\pgfqpoint{2.596184in}{0.489984in}}%
\pgfpathlineto{\pgfqpoint{2.523014in}{0.489984in}}%
\pgfpathlineto{\pgfqpoint{2.397929in}{0.489984in}}%
\pgfpathlineto{\pgfqpoint{2.378916in}{0.489984in}}%
\pgfpathlineto{\pgfqpoint{2.357661in}{0.489984in}}%
\pgfpathlineto{\pgfqpoint{2.333564in}{0.489984in}}%
\pgfpathlineto{\pgfqpoint{2.305746in}{0.489984in}}%
\pgfpathlineto{\pgfqpoint{2.272845in}{0.489984in}}%
\pgfpathlineto{\pgfqpoint{2.232576in}{0.489984in}}%
\pgfpathlineto{\pgfqpoint{2.180661in}{0.489984in}}%
\pgfpathlineto{\pgfqpoint{2.107491in}{0.489984in}}%
\pgfpathlineto{\pgfqpoint{1.982407in}{0.489984in}}%
\pgfpathlineto{\pgfqpoint{1.963393in}{0.489984in}}%
\pgfpathlineto{\pgfqpoint{1.942138in}{0.489984in}}%
\pgfpathlineto{\pgfqpoint{1.918041in}{0.489984in}}%
\pgfpathlineto{\pgfqpoint{1.890224in}{0.489984in}}%
\pgfpathlineto{\pgfqpoint{1.857322in}{0.489984in}}%
\pgfpathlineto{\pgfqpoint{1.817054in}{0.489984in}}%
\pgfpathlineto{\pgfqpoint{1.765139in}{0.489984in}}%
\pgfpathlineto{\pgfqpoint{1.691969in}{0.489984in}}%
\pgfpathlineto{\pgfqpoint{1.566884in}{0.489984in}}%
\pgfpathlineto{\pgfqpoint{1.547871in}{0.489984in}}%
\pgfpathlineto{\pgfqpoint{1.526616in}{0.489984in}}%
\pgfpathlineto{\pgfqpoint{1.502519in}{0.489984in}}%
\pgfpathlineto{\pgfqpoint{1.474701in}{0.489984in}}%
\pgfpathlineto{\pgfqpoint{1.441799in}{0.489984in}}%
\pgfpathlineto{\pgfqpoint{1.401531in}{0.489984in}}%
\pgfpathlineto{\pgfqpoint{1.349616in}{0.489984in}}%
\pgfpathlineto{\pgfqpoint{1.276446in}{0.489984in}}%
\pgfpathlineto{\pgfqpoint{1.151361in}{0.489984in}}%
\pgfpathlineto{\pgfqpoint{1.132348in}{0.489984in}}%
\pgfpathlineto{\pgfqpoint{1.111093in}{0.489984in}}%
\pgfpathlineto{\pgfqpoint{1.086996in}{0.489984in}}%
\pgfpathlineto{\pgfqpoint{1.059178in}{0.489984in}}%
\pgfpathlineto{\pgfqpoint{1.026277in}{0.489984in}}%
\pgfpathlineto{\pgfqpoint{0.986008in}{1.077166in}}%
\pgfpathlineto{\pgfqpoint{0.934093in}{1.095004in}}%
\pgfpathlineto{\pgfqpoint{0.860924in}{1.103882in}}%
\pgfpathlineto{\pgfqpoint{0.735839in}{1.105401in}}%
\pgfpathlineto{\pgfqpoint{0.735839in}{1.105401in}}%
\pgfpathclose%
\pgfusepath{fill}%
\end{pgfscope}%
\begin{pgfscope}%
\pgfpathrectangle{\pgfqpoint{0.642822in}{0.489984in}}{\pgfqpoint{2.046380in}{0.813271in}}%
\pgfusepath{clip}%
\pgfsetbuttcap%
\pgfsetroundjoin%
\definecolor{currentfill}{rgb}{1.000000,0.501961,0.054902}%
\pgfsetfillcolor{currentfill}%
\pgfsetfillopacity{0.500000}%
\pgfsetlinewidth{0.000000pt}%
\definecolor{currentstroke}{rgb}{0.000000,0.000000,0.000000}%
\pgfsetstrokecolor{currentstroke}%
\pgfsetdash{}{0pt}%
\pgfpathmoveto{\pgfqpoint{0.735839in}{1.274857in}}%
\pgfpathlineto{\pgfqpoint{0.735839in}{1.102205in}}%
\pgfpathlineto{\pgfqpoint{0.860924in}{1.087435in}}%
\pgfpathlineto{\pgfqpoint{0.934093in}{1.078409in}}%
\pgfpathlineto{\pgfqpoint{0.986008in}{1.047760in}}%
\pgfpathlineto{\pgfqpoint{1.026277in}{1.053331in}}%
\pgfpathlineto{\pgfqpoint{1.059178in}{1.063596in}}%
\pgfpathlineto{\pgfqpoint{1.086996in}{1.061172in}}%
\pgfpathlineto{\pgfqpoint{1.111093in}{1.036410in}}%
\pgfpathlineto{\pgfqpoint{1.132348in}{1.052410in}}%
\pgfpathlineto{\pgfqpoint{1.151361in}{1.047800in}}%
\pgfpathlineto{\pgfqpoint{1.276446in}{1.022377in}}%
\pgfpathlineto{\pgfqpoint{1.349616in}{1.012669in}}%
\pgfpathlineto{\pgfqpoint{1.401531in}{1.012771in}}%
\pgfpathlineto{\pgfqpoint{1.441799in}{1.001101in}}%
\pgfpathlineto{\pgfqpoint{1.474701in}{0.998651in}}%
\pgfpathlineto{\pgfqpoint{1.502519in}{0.992396in}}%
\pgfpathlineto{\pgfqpoint{1.526616in}{0.988529in}}%
\pgfpathlineto{\pgfqpoint{1.547871in}{0.984926in}}%
\pgfpathlineto{\pgfqpoint{1.566884in}{0.978509in}}%
\pgfpathlineto{\pgfqpoint{1.691969in}{0.957244in}}%
\pgfpathlineto{\pgfqpoint{1.765139in}{0.947802in}}%
\pgfpathlineto{\pgfqpoint{1.817054in}{0.936314in}}%
\pgfpathlineto{\pgfqpoint{1.857322in}{0.931005in}}%
\pgfpathlineto{\pgfqpoint{1.890224in}{0.926181in}}%
\pgfpathlineto{\pgfqpoint{1.918041in}{0.919850in}}%
\pgfpathlineto{\pgfqpoint{1.942138in}{0.916945in}}%
\pgfpathlineto{\pgfqpoint{1.963393in}{0.912409in}}%
\pgfpathlineto{\pgfqpoint{1.982407in}{0.909739in}}%
\pgfpathlineto{\pgfqpoint{2.107491in}{0.883770in}}%
\pgfpathlineto{\pgfqpoint{2.180661in}{0.875247in}}%
\pgfpathlineto{\pgfqpoint{2.232576in}{0.863740in}}%
\pgfpathlineto{\pgfqpoint{2.272845in}{0.851091in}}%
\pgfpathlineto{\pgfqpoint{2.305746in}{0.845854in}}%
\pgfpathlineto{\pgfqpoint{2.333564in}{0.845804in}}%
\pgfpathlineto{\pgfqpoint{2.357661in}{0.839336in}}%
\pgfpathlineto{\pgfqpoint{2.378916in}{0.830944in}}%
\pgfpathlineto{\pgfqpoint{2.397929in}{0.837439in}}%
\pgfpathlineto{\pgfqpoint{2.523014in}{0.806212in}}%
\pgfpathlineto{\pgfqpoint{2.596184in}{0.794947in}}%
\pgfpathlineto{\pgfqpoint{2.596184in}{0.820802in}}%
\pgfpathlineto{\pgfqpoint{2.596184in}{0.820802in}}%
\pgfpathlineto{\pgfqpoint{2.523014in}{0.830117in}}%
\pgfpathlineto{\pgfqpoint{2.397929in}{0.856349in}}%
\pgfpathlineto{\pgfqpoint{2.378916in}{0.864162in}}%
\pgfpathlineto{\pgfqpoint{2.357661in}{0.866390in}}%
\pgfpathlineto{\pgfqpoint{2.333564in}{0.867884in}}%
\pgfpathlineto{\pgfqpoint{2.305746in}{0.873036in}}%
\pgfpathlineto{\pgfqpoint{2.272845in}{0.879112in}}%
\pgfpathlineto{\pgfqpoint{2.232576in}{0.882894in}}%
\pgfpathlineto{\pgfqpoint{2.180661in}{0.896710in}}%
\pgfpathlineto{\pgfqpoint{2.107491in}{0.908612in}}%
\pgfpathlineto{\pgfqpoint{1.982407in}{0.932866in}}%
\pgfpathlineto{\pgfqpoint{1.963393in}{0.934157in}}%
\pgfpathlineto{\pgfqpoint{1.942138in}{0.934626in}}%
\pgfpathlineto{\pgfqpoint{1.918041in}{0.941204in}}%
\pgfpathlineto{\pgfqpoint{1.890224in}{0.943676in}}%
\pgfpathlineto{\pgfqpoint{1.857322in}{0.951997in}}%
\pgfpathlineto{\pgfqpoint{1.817054in}{0.960590in}}%
\pgfpathlineto{\pgfqpoint{1.765139in}{0.971217in}}%
\pgfpathlineto{\pgfqpoint{1.691969in}{0.980070in}}%
\pgfpathlineto{\pgfqpoint{1.566884in}{1.007963in}}%
\pgfpathlineto{\pgfqpoint{1.547871in}{1.011225in}}%
\pgfpathlineto{\pgfqpoint{1.526616in}{1.013030in}}%
\pgfpathlineto{\pgfqpoint{1.502519in}{1.029343in}}%
\pgfpathlineto{\pgfqpoint{1.474701in}{1.028321in}}%
\pgfpathlineto{\pgfqpoint{1.441799in}{1.038499in}}%
\pgfpathlineto{\pgfqpoint{1.401531in}{1.050888in}}%
\pgfpathlineto{\pgfqpoint{1.349616in}{1.050406in}}%
\pgfpathlineto{\pgfqpoint{1.276446in}{1.062800in}}%
\pgfpathlineto{\pgfqpoint{1.151361in}{1.101077in}}%
\pgfpathlineto{\pgfqpoint{1.132348in}{1.128253in}}%
\pgfpathlineto{\pgfqpoint{1.111093in}{1.115218in}}%
\pgfpathlineto{\pgfqpoint{1.086996in}{1.124576in}}%
\pgfpathlineto{\pgfqpoint{1.059178in}{1.136727in}}%
\pgfpathlineto{\pgfqpoint{1.026277in}{1.151475in}}%
\pgfpathlineto{\pgfqpoint{0.986008in}{1.166016in}}%
\pgfpathlineto{\pgfqpoint{0.934093in}{1.169312in}}%
\pgfpathlineto{\pgfqpoint{0.860924in}{1.200323in}}%
\pgfpathlineto{\pgfqpoint{0.735839in}{1.274857in}}%
\pgfpathlineto{\pgfqpoint{0.735839in}{1.274857in}}%
\pgfpathclose%
\pgfusepath{fill}%
\end{pgfscope}%
\begin{pgfscope}%
\pgfsetbuttcap%
\pgfsetroundjoin%
\definecolor{currentfill}{rgb}{0.000000,0.000000,0.000000}%
\pgfsetfillcolor{currentfill}%
\pgfsetlinewidth{0.803000pt}%
\definecolor{currentstroke}{rgb}{0.000000,0.000000,0.000000}%
\pgfsetstrokecolor{currentstroke}%
\pgfsetdash{}{0pt}%
\pgfsys@defobject{currentmarker}{\pgfqpoint{0.000000in}{-0.048611in}}{\pgfqpoint{0.000000in}{0.000000in}}{%
\pgfpathmoveto{\pgfqpoint{0.000000in}{0.000000in}}%
\pgfpathlineto{\pgfqpoint{0.000000in}{-0.048611in}}%
\pgfusepath{stroke,fill}%
}%
\begin{pgfscope}%
\pgfsys@transformshift{1.151361in}{0.489984in}%
\pgfsys@useobject{currentmarker}{}%
\end{pgfscope}%
\end{pgfscope}%
\begin{pgfscope}%
\definecolor{textcolor}{rgb}{0.000000,0.000000,0.000000}%
\pgfsetstrokecolor{textcolor}%
\pgfsetfillcolor{textcolor}%
\pgftext[x=1.151361in,y=0.392761in,,top]{\color{textcolor}\rmfamily\fontsize{9.000000}{10.800000}\selectfont \(\displaystyle {10^{1}}\)}%
\end{pgfscope}%
\begin{pgfscope}%
\pgfsetbuttcap%
\pgfsetroundjoin%
\definecolor{currentfill}{rgb}{0.000000,0.000000,0.000000}%
\pgfsetfillcolor{currentfill}%
\pgfsetlinewidth{0.803000pt}%
\definecolor{currentstroke}{rgb}{0.000000,0.000000,0.000000}%
\pgfsetstrokecolor{currentstroke}%
\pgfsetdash{}{0pt}%
\pgfsys@defobject{currentmarker}{\pgfqpoint{0.000000in}{-0.048611in}}{\pgfqpoint{0.000000in}{0.000000in}}{%
\pgfpathmoveto{\pgfqpoint{0.000000in}{0.000000in}}%
\pgfpathlineto{\pgfqpoint{0.000000in}{-0.048611in}}%
\pgfusepath{stroke,fill}%
}%
\begin{pgfscope}%
\pgfsys@transformshift{1.982407in}{0.489984in}%
\pgfsys@useobject{currentmarker}{}%
\end{pgfscope}%
\end{pgfscope}%
\begin{pgfscope}%
\definecolor{textcolor}{rgb}{0.000000,0.000000,0.000000}%
\pgfsetstrokecolor{textcolor}%
\pgfsetfillcolor{textcolor}%
\pgftext[x=1.982407in,y=0.392761in,,top]{\color{textcolor}\rmfamily\fontsize{9.000000}{10.800000}\selectfont \(\displaystyle {10^{3}}\)}%
\end{pgfscope}%
\begin{pgfscope}%
\definecolor{textcolor}{rgb}{0.000000,0.000000,0.000000}%
\pgfsetstrokecolor{textcolor}%
\pgfsetfillcolor{textcolor}%
\pgftext[x=1.666011in,y=0.223457in,,top]{\color{textcolor}\rmfamily\fontsize{10.000000}{12.000000}\selectfont Number of training examples}%
\end{pgfscope}%
\begin{pgfscope}%
\pgfsetbuttcap%
\pgfsetroundjoin%
\definecolor{currentfill}{rgb}{0.000000,0.000000,0.000000}%
\pgfsetfillcolor{currentfill}%
\pgfsetlinewidth{0.803000pt}%
\definecolor{currentstroke}{rgb}{0.000000,0.000000,0.000000}%
\pgfsetstrokecolor{currentstroke}%
\pgfsetdash{}{0pt}%
\pgfsys@defobject{currentmarker}{\pgfqpoint{-0.048611in}{0.000000in}}{\pgfqpoint{-0.000000in}{0.000000in}}{%
\pgfpathmoveto{\pgfqpoint{-0.000000in}{0.000000in}}%
\pgfpathlineto{\pgfqpoint{-0.048611in}{0.000000in}}%
\pgfusepath{stroke,fill}%
}%
\begin{pgfscope}%
\pgfsys@transformshift{0.642822in}{0.489984in}%
\pgfsys@useobject{currentmarker}{}%
\end{pgfscope}%
\end{pgfscope}%
\begin{pgfscope}%
\definecolor{textcolor}{rgb}{0.000000,0.000000,0.000000}%
\pgfsetstrokecolor{textcolor}%
\pgfsetfillcolor{textcolor}%
\pgftext[x=0.481364in, y=0.446581in, left, base]{\color{textcolor}\rmfamily\fontsize{9.000000}{10.800000}\selectfont \(\displaystyle {0}\)}%
\end{pgfscope}%
\begin{pgfscope}%
\pgfsetbuttcap%
\pgfsetroundjoin%
\definecolor{currentfill}{rgb}{0.000000,0.000000,0.000000}%
\pgfsetfillcolor{currentfill}%
\pgfsetlinewidth{0.803000pt}%
\definecolor{currentstroke}{rgb}{0.000000,0.000000,0.000000}%
\pgfsetstrokecolor{currentstroke}%
\pgfsetdash{}{0pt}%
\pgfsys@defobject{currentmarker}{\pgfqpoint{-0.048611in}{0.000000in}}{\pgfqpoint{-0.000000in}{0.000000in}}{%
\pgfpathmoveto{\pgfqpoint{-0.000000in}{0.000000in}}%
\pgfpathlineto{\pgfqpoint{-0.048611in}{0.000000in}}%
\pgfusepath{stroke,fill}%
}%
\begin{pgfscope}%
\pgfsys@transformshift{0.642822in}{0.709787in}%
\pgfsys@useobject{currentmarker}{}%
\end{pgfscope}%
\end{pgfscope}%
\begin{pgfscope}%
\definecolor{textcolor}{rgb}{0.000000,0.000000,0.000000}%
\pgfsetstrokecolor{textcolor}%
\pgfsetfillcolor{textcolor}%
\pgftext[x=0.279012in, y=0.665065in, left, base]{\color{textcolor}\rmfamily\fontsize{9.000000}{10.800000}\selectfont \(\displaystyle {10^{-2}}\)}%
\end{pgfscope}%
\begin{pgfscope}%
\pgfsetbuttcap%
\pgfsetroundjoin%
\definecolor{currentfill}{rgb}{0.000000,0.000000,0.000000}%
\pgfsetfillcolor{currentfill}%
\pgfsetlinewidth{0.803000pt}%
\definecolor{currentstroke}{rgb}{0.000000,0.000000,0.000000}%
\pgfsetstrokecolor{currentstroke}%
\pgfsetdash{}{0pt}%
\pgfsys@defobject{currentmarker}{\pgfqpoint{-0.048611in}{0.000000in}}{\pgfqpoint{-0.000000in}{0.000000in}}{%
\pgfpathmoveto{\pgfqpoint{-0.000000in}{0.000000in}}%
\pgfpathlineto{\pgfqpoint{-0.048611in}{0.000000in}}%
\pgfusepath{stroke,fill}%
}%
\begin{pgfscope}%
\pgfsys@transformshift{0.642822in}{0.907609in}%
\pgfsys@useobject{currentmarker}{}%
\end{pgfscope}%
\end{pgfscope}%
\begin{pgfscope}%
\definecolor{textcolor}{rgb}{0.000000,0.000000,0.000000}%
\pgfsetstrokecolor{textcolor}%
\pgfsetfillcolor{textcolor}%
\pgftext[x=0.279012in, y=0.862887in, left, base]{\color{textcolor}\rmfamily\fontsize{9.000000}{10.800000}\selectfont \(\displaystyle {10^{-1}}\)}%
\end{pgfscope}%
\begin{pgfscope}%
\pgfsetbuttcap%
\pgfsetroundjoin%
\definecolor{currentfill}{rgb}{0.000000,0.000000,0.000000}%
\pgfsetfillcolor{currentfill}%
\pgfsetlinewidth{0.803000pt}%
\definecolor{currentstroke}{rgb}{0.000000,0.000000,0.000000}%
\pgfsetstrokecolor{currentstroke}%
\pgfsetdash{}{0pt}%
\pgfsys@defobject{currentmarker}{\pgfqpoint{-0.048611in}{0.000000in}}{\pgfqpoint{-0.000000in}{0.000000in}}{%
\pgfpathmoveto{\pgfqpoint{-0.000000in}{0.000000in}}%
\pgfpathlineto{\pgfqpoint{-0.048611in}{0.000000in}}%
\pgfusepath{stroke,fill}%
}%
\begin{pgfscope}%
\pgfsys@transformshift{0.642822in}{1.105432in}%
\pgfsys@useobject{currentmarker}{}%
\end{pgfscope}%
\end{pgfscope}%
\begin{pgfscope}%
\definecolor{textcolor}{rgb}{0.000000,0.000000,0.000000}%
\pgfsetstrokecolor{textcolor}%
\pgfsetfillcolor{textcolor}%
\pgftext[x=0.359258in, y=1.060710in, left, base]{\color{textcolor}\rmfamily\fontsize{9.000000}{10.800000}\selectfont \(\displaystyle {10^{0}}\)}%
\end{pgfscope}%
\begin{pgfscope}%
\pgfsetbuttcap%
\pgfsetroundjoin%
\definecolor{currentfill}{rgb}{0.000000,0.000000,0.000000}%
\pgfsetfillcolor{currentfill}%
\pgfsetlinewidth{0.803000pt}%
\definecolor{currentstroke}{rgb}{0.000000,0.000000,0.000000}%
\pgfsetstrokecolor{currentstroke}%
\pgfsetdash{}{0pt}%
\pgfsys@defobject{currentmarker}{\pgfqpoint{-0.048611in}{0.000000in}}{\pgfqpoint{-0.000000in}{0.000000in}}{%
\pgfpathmoveto{\pgfqpoint{-0.000000in}{0.000000in}}%
\pgfpathlineto{\pgfqpoint{-0.048611in}{0.000000in}}%
\pgfusepath{stroke,fill}%
}%
\begin{pgfscope}%
\pgfsys@transformshift{0.642822in}{1.303255in}%
\pgfsys@useobject{currentmarker}{}%
\end{pgfscope}%
\end{pgfscope}%
\begin{pgfscope}%
\definecolor{textcolor}{rgb}{0.000000,0.000000,0.000000}%
\pgfsetstrokecolor{textcolor}%
\pgfsetfillcolor{textcolor}%
\pgftext[x=0.359258in, y=1.258533in, left, base]{\color{textcolor}\rmfamily\fontsize{9.000000}{10.800000}\selectfont \(\displaystyle {10^{1}}\)}%
\end{pgfscope}%
\begin{pgfscope}%
\definecolor{textcolor}{rgb}{0.000000,0.000000,0.000000}%
\pgfsetstrokecolor{textcolor}%
\pgfsetfillcolor{textcolor}%
\pgftext[x=0.223457in,y=0.896619in,,bottom,rotate=90.000000]{\color{textcolor}\rmfamily\fontsize{10.000000}{12.000000}\selectfont Test MSE}%
\end{pgfscope}%
\begin{pgfscope}%
\pgfpathrectangle{\pgfqpoint{0.642822in}{0.489984in}}{\pgfqpoint{2.046380in}{0.813271in}}%
\pgfusepath{clip}%
\pgfsetrectcap%
\pgfsetroundjoin%
\pgfsetlinewidth{1.505625pt}%
\definecolor{currentstroke}{rgb}{0.000000,0.419608,0.643137}%
\pgfsetstrokecolor{currentstroke}%
\pgfsetdash{}{0pt}%
\pgfpathmoveto{\pgfqpoint{0.735839in}{1.085625in}}%
\pgfpathlineto{\pgfqpoint{0.860924in}{1.071001in}}%
\pgfpathlineto{\pgfqpoint{0.934093in}{1.023416in}}%
\pgfpathlineto{\pgfqpoint{0.986008in}{0.949834in}}%
\pgfpathlineto{\pgfqpoint{1.026277in}{0.489984in}}%
\pgfpathlineto{\pgfqpoint{1.059178in}{0.489984in}}%
\pgfpathlineto{\pgfqpoint{1.086996in}{0.489984in}}%
\pgfpathlineto{\pgfqpoint{1.111093in}{0.489984in}}%
\pgfpathlineto{\pgfqpoint{1.132348in}{0.489984in}}%
\pgfpathlineto{\pgfqpoint{1.151361in}{0.489984in}}%
\pgfpathlineto{\pgfqpoint{1.276446in}{0.489984in}}%
\pgfpathlineto{\pgfqpoint{1.349616in}{0.489984in}}%
\pgfpathlineto{\pgfqpoint{1.401531in}{0.489984in}}%
\pgfpathlineto{\pgfqpoint{1.441799in}{0.489984in}}%
\pgfpathlineto{\pgfqpoint{1.474701in}{0.489984in}}%
\pgfpathlineto{\pgfqpoint{1.502519in}{0.489984in}}%
\pgfpathlineto{\pgfqpoint{1.526616in}{0.489984in}}%
\pgfpathlineto{\pgfqpoint{1.547871in}{0.489984in}}%
\pgfpathlineto{\pgfqpoint{1.566884in}{0.489984in}}%
\pgfpathlineto{\pgfqpoint{1.691969in}{0.489984in}}%
\pgfpathlineto{\pgfqpoint{1.765139in}{0.489984in}}%
\pgfpathlineto{\pgfqpoint{1.817054in}{0.489984in}}%
\pgfpathlineto{\pgfqpoint{1.857322in}{0.489984in}}%
\pgfpathlineto{\pgfqpoint{1.890224in}{0.489984in}}%
\pgfpathlineto{\pgfqpoint{1.918041in}{0.489984in}}%
\pgfpathlineto{\pgfqpoint{1.942138in}{0.489984in}}%
\pgfpathlineto{\pgfqpoint{1.963393in}{0.489984in}}%
\pgfpathlineto{\pgfqpoint{1.982407in}{0.489984in}}%
\pgfpathlineto{\pgfqpoint{2.107491in}{0.489984in}}%
\pgfpathlineto{\pgfqpoint{2.180661in}{0.489984in}}%
\pgfpathlineto{\pgfqpoint{2.232576in}{0.489984in}}%
\pgfpathlineto{\pgfqpoint{2.272845in}{0.489984in}}%
\pgfpathlineto{\pgfqpoint{2.305746in}{0.489984in}}%
\pgfpathlineto{\pgfqpoint{2.333564in}{0.489984in}}%
\pgfpathlineto{\pgfqpoint{2.357661in}{0.489984in}}%
\pgfpathlineto{\pgfqpoint{2.378916in}{0.489984in}}%
\pgfpathlineto{\pgfqpoint{2.397929in}{0.489984in}}%
\pgfpathlineto{\pgfqpoint{2.523014in}{0.489984in}}%
\pgfpathlineto{\pgfqpoint{2.596184in}{0.489984in}}%
\pgfusepath{stroke}%
\end{pgfscope}%
\begin{pgfscope}%
\pgfpathrectangle{\pgfqpoint{0.642822in}{0.489984in}}{\pgfqpoint{2.046380in}{0.813271in}}%
\pgfusepath{clip}%
\pgfsetrectcap%
\pgfsetroundjoin%
\pgfsetlinewidth{1.505625pt}%
\definecolor{currentstroke}{rgb}{1.000000,0.501961,0.054902}%
\pgfsetstrokecolor{currentstroke}%
\pgfsetdash{}{0pt}%
\pgfpathmoveto{\pgfqpoint{0.735839in}{1.156246in}}%
\pgfpathlineto{\pgfqpoint{0.860924in}{1.125035in}}%
\pgfpathlineto{\pgfqpoint{0.934093in}{1.115539in}}%
\pgfpathlineto{\pgfqpoint{0.986008in}{1.106286in}}%
\pgfpathlineto{\pgfqpoint{1.026277in}{1.093296in}}%
\pgfpathlineto{\pgfqpoint{1.059178in}{1.090725in}}%
\pgfpathlineto{\pgfqpoint{1.086996in}{1.086647in}}%
\pgfpathlineto{\pgfqpoint{1.111093in}{1.078714in}}%
\pgfpathlineto{\pgfqpoint{1.132348in}{1.075971in}}%
\pgfpathlineto{\pgfqpoint{1.151361in}{1.070922in}}%
\pgfpathlineto{\pgfqpoint{1.276446in}{1.046994in}}%
\pgfpathlineto{\pgfqpoint{1.349616in}{1.034987in}}%
\pgfpathlineto{\pgfqpoint{1.401531in}{1.023462in}}%
\pgfpathlineto{\pgfqpoint{1.441799in}{1.014984in}}%
\pgfpathlineto{\pgfqpoint{1.474701in}{1.008978in}}%
\pgfpathlineto{\pgfqpoint{1.502519in}{1.006639in}}%
\pgfpathlineto{\pgfqpoint{1.526616in}{1.000198in}}%
\pgfpathlineto{\pgfqpoint{1.547871in}{0.996566in}}%
\pgfpathlineto{\pgfqpoint{1.566884in}{0.993637in}}%
\pgfpathlineto{\pgfqpoint{1.691969in}{0.970926in}}%
\pgfpathlineto{\pgfqpoint{1.765139in}{0.958274in}}%
\pgfpathlineto{\pgfqpoint{1.817054in}{0.951227in}}%
\pgfpathlineto{\pgfqpoint{1.857322in}{0.941105in}}%
\pgfpathlineto{\pgfqpoint{1.890224in}{0.936158in}}%
\pgfpathlineto{\pgfqpoint{1.918041in}{0.929308in}}%
\pgfpathlineto{\pgfqpoint{1.942138in}{0.924939in}}%
\pgfpathlineto{\pgfqpoint{1.963393in}{0.921568in}}%
\pgfpathlineto{\pgfqpoint{1.982407in}{0.918094in}}%
\pgfpathlineto{\pgfqpoint{2.107491in}{0.895360in}}%
\pgfpathlineto{\pgfqpoint{2.180661in}{0.882641in}}%
\pgfpathlineto{\pgfqpoint{2.232576in}{0.872684in}}%
\pgfpathlineto{\pgfqpoint{2.272845in}{0.866906in}}%
\pgfpathlineto{\pgfqpoint{2.305746in}{0.860854in}}%
\pgfpathlineto{\pgfqpoint{2.333564in}{0.855659in}}%
\pgfpathlineto{\pgfqpoint{2.357661in}{0.851690in}}%
\pgfpathlineto{\pgfqpoint{2.378916in}{0.847999in}}%
\pgfpathlineto{\pgfqpoint{2.397929in}{0.844590in}}%
\pgfpathlineto{\pgfqpoint{2.523014in}{0.821469in}}%
\pgfpathlineto{\pgfqpoint{2.596184in}{0.808411in}}%
\pgfusepath{stroke}%
\end{pgfscope}%
\begin{pgfscope}%
\pgfsetrectcap%
\pgfsetmiterjoin%
\pgfsetlinewidth{0.803000pt}%
\definecolor{currentstroke}{rgb}{0.000000,0.000000,0.000000}%
\pgfsetstrokecolor{currentstroke}%
\pgfsetdash{}{0pt}%
\pgfpathmoveto{\pgfqpoint{0.642822in}{0.489984in}}%
\pgfpathlineto{\pgfqpoint{0.642822in}{1.303255in}}%
\pgfusepath{stroke}%
\end{pgfscope}%
\begin{pgfscope}%
\pgfsetrectcap%
\pgfsetmiterjoin%
\pgfsetlinewidth{0.803000pt}%
\definecolor{currentstroke}{rgb}{0.000000,0.000000,0.000000}%
\pgfsetstrokecolor{currentstroke}%
\pgfsetdash{}{0pt}%
\pgfpathmoveto{\pgfqpoint{2.689201in}{0.489984in}}%
\pgfpathlineto{\pgfqpoint{2.689201in}{1.303255in}}%
\pgfusepath{stroke}%
\end{pgfscope}%
\begin{pgfscope}%
\pgfsetrectcap%
\pgfsetmiterjoin%
\pgfsetlinewidth{0.803000pt}%
\definecolor{currentstroke}{rgb}{0.000000,0.000000,0.000000}%
\pgfsetstrokecolor{currentstroke}%
\pgfsetdash{}{0pt}%
\pgfpathmoveto{\pgfqpoint{0.642822in}{0.489984in}}%
\pgfpathlineto{\pgfqpoint{2.689201in}{0.489984in}}%
\pgfusepath{stroke}%
\end{pgfscope}%
\begin{pgfscope}%
\pgfsetrectcap%
\pgfsetmiterjoin%
\pgfsetlinewidth{0.803000pt}%
\definecolor{currentstroke}{rgb}{0.000000,0.000000,0.000000}%
\pgfsetstrokecolor{currentstroke}%
\pgfsetdash{}{0pt}%
\pgfpathmoveto{\pgfqpoint{0.642822in}{1.303255in}}%
\pgfpathlineto{\pgfqpoint{2.689201in}{1.303255in}}%
\pgfusepath{stroke}%
\end{pgfscope}%
\begin{pgfscope}%
\pgfsetbuttcap%
\pgfsetmiterjoin%
\definecolor{currentfill}{rgb}{1.000000,1.000000,1.000000}%
\pgfsetfillcolor{currentfill}%
\pgfsetfillopacity{0.800000}%
\pgfsetlinewidth{1.003750pt}%
\definecolor{currentstroke}{rgb}{0.800000,0.800000,0.800000}%
\pgfsetstrokecolor{currentstroke}%
\pgfsetstrokeopacity{0.800000}%
\pgfsetdash{}{0pt}%
\pgfpathmoveto{\pgfqpoint{0.730322in}{1.365755in}}%
\pgfpathlineto{\pgfqpoint{2.673903in}{1.365755in}}%
\pgfpathquadraticcurveto{\pgfqpoint{2.698903in}{1.365755in}}{\pgfqpoint{2.698903in}{1.390755in}}%
\pgfpathlineto{\pgfqpoint{2.698903in}{1.768856in}}%
\pgfpathquadraticcurveto{\pgfqpoint{2.698903in}{1.793856in}}{\pgfqpoint{2.673903in}{1.793856in}}%
\pgfpathlineto{\pgfqpoint{0.730322in}{1.793856in}}%
\pgfpathquadraticcurveto{\pgfqpoint{0.705322in}{1.793856in}}{\pgfqpoint{0.705322in}{1.768856in}}%
\pgfpathlineto{\pgfqpoint{0.705322in}{1.390755in}}%
\pgfpathquadraticcurveto{\pgfqpoint{0.705322in}{1.365755in}}{\pgfqpoint{0.730322in}{1.365755in}}%
\pgfpathlineto{\pgfqpoint{0.730322in}{1.365755in}}%
\pgfpathclose%
\pgfusepath{stroke,fill}%
\end{pgfscope}%
\begin{pgfscope}%
\pgfsetrectcap%
\pgfsetroundjoin%
\pgfsetlinewidth{1.505625pt}%
\definecolor{currentstroke}{rgb}{0.000000,0.419608,0.643137}%
\pgfsetstrokecolor{currentstroke}%
\pgfsetdash{}{0pt}%
\pgfpathmoveto{\pgfqpoint{0.755322in}{1.700106in}}%
\pgfpathlineto{\pgfqpoint{0.880322in}{1.700106in}}%
\pgfpathlineto{\pgfqpoint{1.005322in}{1.700106in}}%
\pgfusepath{stroke}%
\end{pgfscope}%
\begin{pgfscope}%
\definecolor{textcolor}{rgb}{0.000000,0.000000,0.000000}%
\pgfsetstrokecolor{textcolor}%
\pgfsetfillcolor{textcolor}%
\pgftext[x=1.105322in,y=1.656356in,left,base]{\color{textcolor}\rmfamily\fontsize{9.000000}{10.800000}\selectfont Least-squares}%
\end{pgfscope}%
\begin{pgfscope}%
\pgfsetbuttcap%
\pgfsetroundjoin%
\pgfsetlinewidth{1.003750pt}%
\definecolor{currentstroke}{rgb}{0.000000,0.000000,0.000000}%
\pgfsetstrokecolor{currentstroke}%
\pgfsetdash{{3.700000pt}{1.600000pt}}{0.000000pt}%
\pgfpathmoveto{\pgfqpoint{0.755322in}{1.490755in}}%
\pgfpathlineto{\pgfqpoint{0.880322in}{1.490755in}}%
\pgfpathlineto{\pgfqpoint{1.005322in}{1.490755in}}%
\pgfusepath{stroke}%
\end{pgfscope}%
\begin{pgfscope}%
\definecolor{textcolor}{rgb}{0.000000,0.000000,0.000000}%
\pgfsetstrokecolor{textcolor}%
\pgfsetfillcolor{textcolor}%
\pgftext[x=1.105322in,y=1.447005in,left,base]{\color{textcolor}\rmfamily\fontsize{9.000000}{10.800000}\selectfont \(\displaystyle 2 \times n^{-2/d} \; (d = 5)\)}%
\end{pgfscope}%
\begin{pgfscope}%
\pgfsetrectcap%
\pgfsetroundjoin%
\pgfsetlinewidth{1.505625pt}%
\definecolor{currentstroke}{rgb}{1.000000,0.501961,0.054902}%
\pgfsetstrokecolor{currentstroke}%
\pgfsetdash{}{0pt}%
\pgfpathmoveto{\pgfqpoint{1.999176in}{1.700106in}}%
\pgfpathlineto{\pgfqpoint{2.124176in}{1.700106in}}%
\pgfpathlineto{\pgfqpoint{2.249176in}{1.700106in}}%
\pgfusepath{stroke}%
\end{pgfscope}%
\begin{pgfscope}%
\definecolor{textcolor}{rgb}{0.000000,0.000000,0.000000}%
\pgfsetstrokecolor{textcolor}%
\pgfsetfillcolor{textcolor}%
\pgftext[x=2.349176in,y=1.656356in,left,base]{\color{textcolor}\rmfamily\fontsize{9.000000}{10.800000}\selectfont \(\displaystyle 1\)-NN}%
\end{pgfscope}%
\begin{pgfscope}%
\pgfsetrectcap%
\pgfsetroundjoin%
\pgfsetlinewidth{1.505625pt}%
\definecolor{currentstroke}{rgb}{0.670588,0.670588,0.670588}%
\pgfsetstrokecolor{currentstroke}%
\pgfsetstrokeopacity{0.000000}%
\pgfsetdash{}{0pt}%
\pgfpathmoveto{\pgfqpoint{1.999176in}{1.525800in}}%
\pgfpathlineto{\pgfqpoint{2.124176in}{1.525800in}}%
\pgfpathlineto{\pgfqpoint{2.249176in}{1.525800in}}%
\pgfusepath{stroke}%
\end{pgfscope}%
\begin{pgfscope}%
\definecolor{textcolor}{rgb}{0.000000,0.000000,0.000000}%
\pgfsetstrokecolor{textcolor}%
\pgfsetfillcolor{textcolor}%
\pgftext[x=2.349176in,y=1.482050in,left,base]{\color{textcolor}\rmfamily\fontsize{9.000000}{10.800000}\selectfont \(\displaystyle \quad\)}%
\end{pgfscope}%
\begin{pgfscope}%
\pgfpathrectangle{\pgfqpoint{0.642822in}{0.489984in}}{\pgfqpoint{2.046380in}{0.813271in}}%
\pgfusepath{clip}%
\pgfsetbuttcap%
\pgfsetroundjoin%
\pgfsetlinewidth{1.003750pt}%
\definecolor{currentstroke}{rgb}{0.000000,0.000000,0.000000}%
\pgfsetstrokecolor{currentstroke}%
\pgfsetdash{{3.700000pt}{1.600000pt}}{0.000000pt}%
\pgfpathmoveto{\pgfqpoint{0.735839in}{1.164982in}}%
\pgfpathlineto{\pgfqpoint{2.596178in}{0.810713in}}%
\pgfpathlineto{\pgfqpoint{2.596178in}{0.810713in}}%
\pgfusepath{stroke}%
\end{pgfscope}%
\end{pgfpicture}%
\makeatother%
\endgroup%

%% file: tex/harmonics.tex
\subsection{Cliff-learning and regularization}
\label{sec:nn-harmonic-cliff}

In this section, we train neural networks to learn 2D bandlimited periodic functions of the form
\begin{equation}
    \label{eqn:harmonic-fn-form}
    h(x) =
    \sum_{a, b \in \{-B, \ldots, B\}}
    \alpha_{ij} \cos(2\pi (ax_1 + bx_2))
  + \beta_{ij} \sin(2\pi (ax_1 + bx_2)),
\end{equation}
where $h : [0, 1]^2 \to \mathbb{R}$.
We train our networks to minimize the mean-squared error on datasets $\{(x_i, h(x_i))\}_{i=1}^n$ where the $x_i$'s are sampled i.i.d. from $\mathrm{Unif}([0, 1]^2)$.

When trained to minimize MSE loss, our networks obey power-law loss scaling. However, in principle, sampling theory tells us that we can reconstruct $h$ perfectly from just $(2B + 1)^2$ distinct samples. We show in \Cref{fig:nn-harmonic-cliff} it is possible to make neural networks cliff have a soft-cliff at this threshold by regularizing their high-frequency components (see \Cref{app:mcls-reg}).

Regularization can be considered the dual form of a prior. The fact that we get cliff-learning with regularization, and power-law scaling without regularization, supports our view that cliff-learning arises from prior-task compatibility.

\begin{figure}
    \vspace{-6mm}
    \centering
    \import{res/harmonics/}{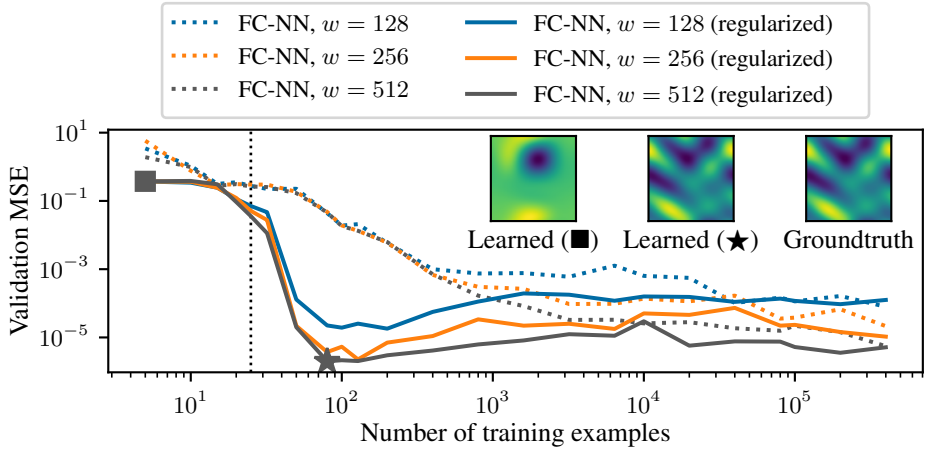}
    \vspace{-3mm}
    \caption{Data-scaling curves for a fully connected ReLU network with 3 hidden layers trained using Adam to learn a 2D harmonic function from \Cref{sec:nn-harmonic-cliff} with bandlimit $B=2$ (visualized in the top right). Without regularization, the data-scaling laws follow a power law trend, plateauing at a level governed by network size. With bandwidth regularization, we get cliff-learning around the theoretical limit of $n = (2B + 1)^2$ (vertical dotted line). Why our method has slightly non-monotonic scaling after the initial cliff is an open question.
    }
    \label{fig:nn-harmonic-cliff}
    \vspace{-3mm}
\end{figure}

%% file: tex/related-work.tex
\section{Related work}
\label{app:related}

\citet{hernandez2021scaling} studies scaling laws for transfer learning and concludes that power law data-scaling holds for transfer learning. However the experiments in \citet{hernandez2021scaling} are conducted exclusively in the large-$n$ regime (even the smallest experiments use $>3 \times 10^5$ characters of downstream text data). In our experiments, we observed that cliff-learning is a transient phenomenon and that data-scaling reverts to a power law in the large-$n$ regime. We believe the absence of cliff-like behavior in \citet{hernandez2021scaling} is caused by working in a large-$n$ regime where cliff-learning does not apply.

\citet{arora2019theoretical} proposes a theoretical model that explains why contrastive-learning produces representations that support well-performing linear probes. However, \citet{arora2019theoretical} provides limited analysis of the data-efficiency of learning such probes, and their theoretical model is also not fully applicable to CLIP-style models.

\citet{sorscher2022neural} studies the data-scaling of linear probes for foundation models and macaque inferotemporal cortex representations.

While previous work~\citep{hestness2017deep,rosenfeld2019constructive} has also found that power laws do not hold in the low-data regime, our work is the first to examine the transfer learning setting in this regime, and is the first to identify the cliff-learning phenomenon.

\citet{sorscher2022beyond} proposes data pruning as a method for beating power law scaling. However, while data-pruning reduces the amount of data that is trained on, it still requires a large corpus of data on which pruning is performed. Thus data-pruning improves compute scaling but does not reduce the net amount of data required to reach a given loss.\footnote{For example, data pruning on its own can not address the issue of large language models requiring more data than is easily obtainable from the internet, c.f. \citet{nostalgebraist2022implications}.} Moreover, the empirical performance of the data-pruning techniques demonstrated \citet{sorscher2022beyond} are greatly outperformed by the cliff-learning we observe for foundational model transfer learning (\Cref{fig:data-pruning-comparison-cifar10}). Nonetheless, the ideas in \citet{sorscher2022beyond} could plausibly be extended into the active-learning or reinforcement learning setting, where intelligent selection of data / environment exploration could yield true improvements in data-scaling.

Finally \citet{viering2021shape} surveys prior work on cliff-learning\footnote{\citet{viering2021shape} refers to cliff-learning as a ``phase-transition'' in data-scaling.} for both humans and machines. However, none of the surveyed works demonstrate cliff-learning in the context of deep learning. Compared to \citet{viering2021shape} our main contribution is to provide a unifying framework of prior/task compatibility to explain cliff-learning, and to demonstrate that it can be achieved with deep-neural networks in practical scenarios via transfer learning on top of foundation models \citep{bommasani2021opportunities}.

%% file: tex/appendix-misc.tex
\section{Concavity on a log-log plot}
\label{app:concavity}

We provide a short proof that any scaling law of the form $f(n) = A \times n^{-\alpha} + E$ is non-concave when plotted on log-log scale. Thus, any scaling law which is concave must  (locally) scale faster than any power law.

\begin{proposition}
    For $A, \alpha, E \geq 0$,
    a power law of the form $f(n) = A \times n^{-\alpha} + E$ is non-concave when plotted on a log-log scale.
\end{proposition}
\begin{proof}
    If $A$ or $\alpha$ are zero, $f(n)$ is a constant and non-concave.
    Otherwise,
    \begin{equation*}
        \frac{\partial^2}{\partial x^2} \log(f(e^x))
        = \log (A \times (e^x)^{-\alpha} + E)
        = \frac{
            \alpha^2 A E e^{\alpha x}
        }{
            \left(A + E e^{\alpha x}\right)^2
        } \geq 0.
        \qedhere
    \end{equation*}
\end{proof}

\clearpage
\section{Experimental details}

\begin{itemize}
    \item 
    Linear probes were trained using scikit-learn~\citep{scikit-learn} with $C=100$ as the regularization parameter (\Cref{fig:data-pruning-comparison-cifar10,fig:transfer-deep-dive,fig:cliff-failures}).

    \item
    With the exception of linear probes,
    we used the Adam optimizer~\citep{kingma2014adam} to train / finetune neural networks.

    \item
    For full model finetuning, we used a validation set for early stopping, which strictly speaking should be counted as part of the training dataset. We choose to omit the validation set from the training dataset size since we believe it is possible to find a smarter mechanism for model selection / early stopping that does not require a large validation set.

    On CIFAR-10 and SVHN (\Cref{fig:transfer-deep-dive,fig:cliff-failures}), the validation set was of size $10^4$ and this set was used both for early stopping and for checkpoint selection (the checkpoint with the lowest validation loss was selected).

    On the harmonic function learning task (\Cref{fig:nn-harmonic-cliff}), a validation set of size $512$ was used for early-stopping, and the final checkpoint right before stopping was selected for final evaluation. The same validation set was also used as the test set.

    \item
    Curves for the cited papers are extracted from original publication figures via \href{https://github.com/ankitrohatgi/WebPlotDigitizer}{WebPlotDigitizer} (\Cref{fig:data-pruning-comparison-cifar10,fig:transfer-deep-dive,fig:cliff-failures}).

    \item
    Data-scaling experimental data was tracked using Weights \& Biases~\citep{wandb}.
\end{itemize}

\clearpage
\section{More foundation model transfer learning results}

\subsection{Foundation models do not always yield cliff-learning}
\label{sec:cliff-failures}
We show in \Cref{fig:cliff-failures,fig:svhn-deep-dive} that foundation-model transfer learning does not always yield substantial cliff-learning.

\subsection{More linear probe data-scaling}
We give plot more data-scaling laws for foundation model linear probes in \Cref{fig:all-the-linear-probes}.

\begin{figure}[h]
    \vspace{2cm}
    \centering
    \import{res/transfer}{cliff-failures.pgf}
    \vspace{-3mm}
    \caption{
        Data-scaling laws for deep neural networks on two standard image classification benchmarks. Some but not all of the laws exhibit cliff-scaling. Shading captures the range of measured values across 50 different trials, and solid lines denote the median over trials. Humans achieve 4.57\% error on CIFAR-10~\citep{ho2018cifar10} and 2\% error on SVHN~\citep{netzer2011reading}.
        \textbf{Left:}
        A \texttt{ViT} CLIP model has a large region of cliff-learning, spanning from random-guessing to human parity. However, foundation model finetuning does not always results in a cliff, as demonstrated by a BeIT linear probe~\citep{bao2021beit}.
        \textbf{Right:}
        While a CLIP linear-probe aggressively cliffs on CIFAR10, it behaves much closer to a power-law on SVHN. Finetuning the whole CLIP model yields a curve with similar asymptotics to from-scratch learning, though we can infer the finetuning curve must cliff more in the small-$n$ regime (since error cannot exceed 100\%).
    }
    \label{fig:cliff-failures}
    \vspace{-3mm}
\end{figure}

\begin{figure}
    \centering
    \import{res/transfer}{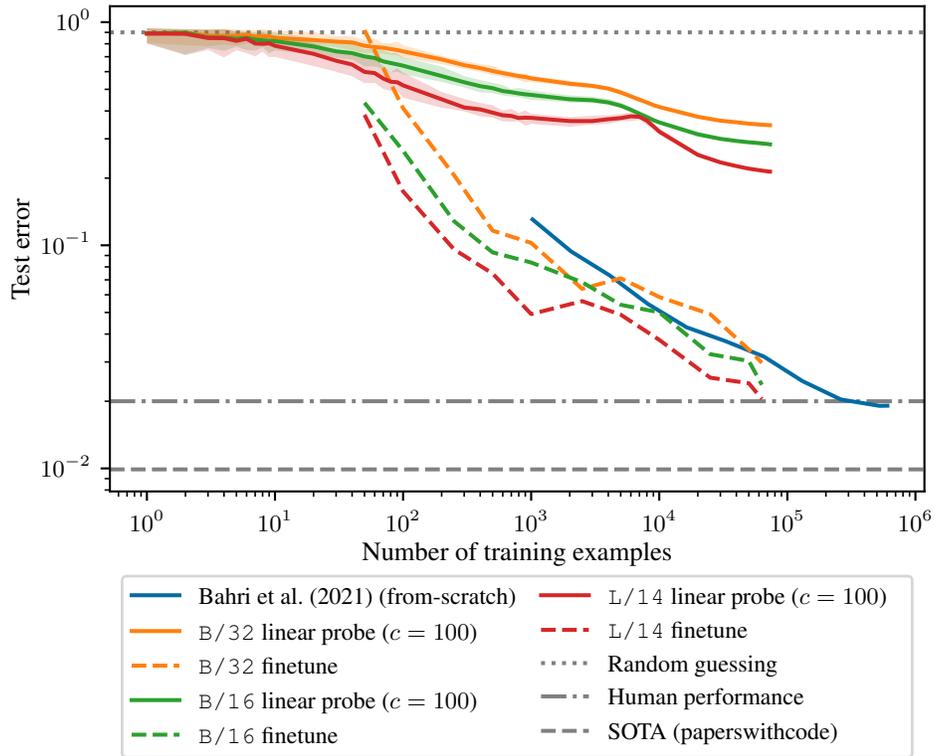}
    \vspace{-3mm}
    \caption{
        Data-scaling of different finetuning methods on SVHN for three OpenAI CLIP models of increasing quality: $\texttt{B/32} < \texttt{B/16} < \texttt{L/14}$.
    }
    \label{fig:svhn-deep-dive}
\end{figure}

\begin{figure}
    \centering
    \import{res/transfer}{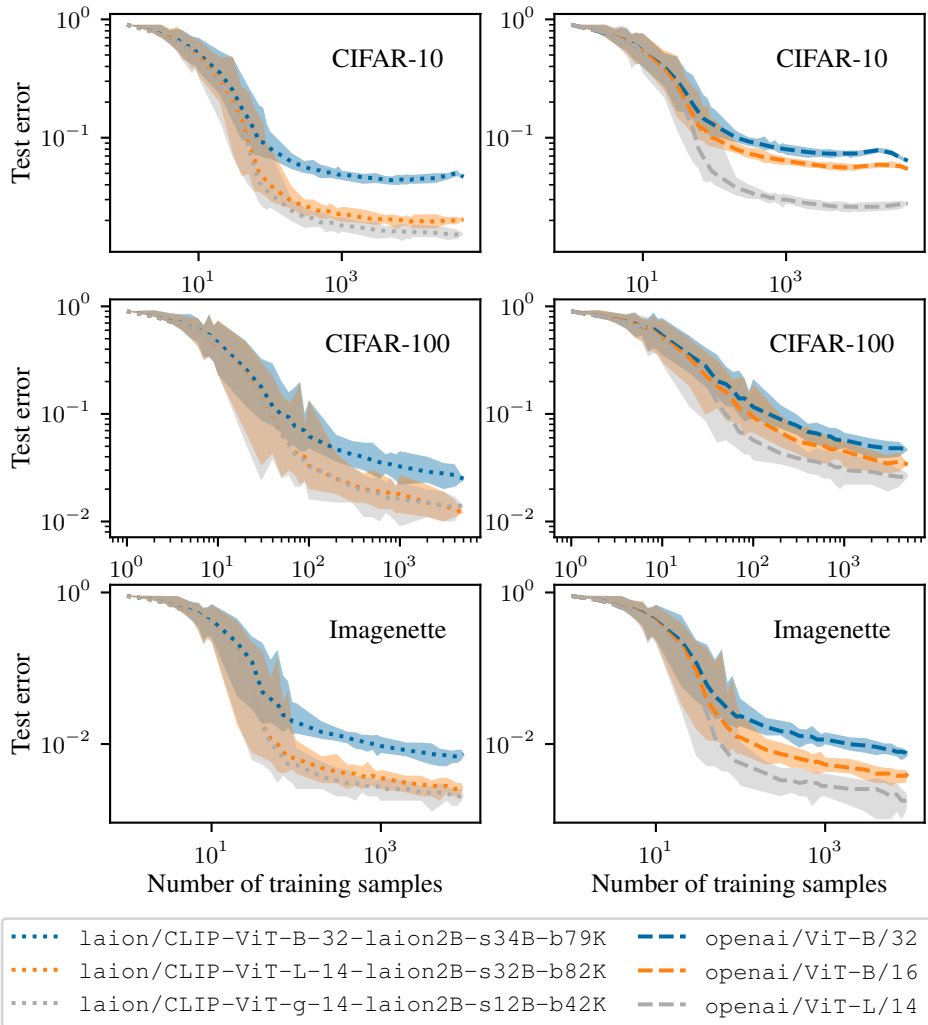}
    \vspace{-3mm}
    \caption{
        Linear probe data-scaling of OpenAI and LAION CLIP models on CIFAR-10, SVHN, and ImageNette~\citep{imagewang}. Models consistently maintain relative ordering across datasets and $n$.
    }
    \label{fig:all-the-linear-probes}
\end{figure}

\clearpage
\section{Monte-Carlo bandwidth regularization}
\label{app:mcls-reg}

In this section, we describe how we can can regularize the high-frequency components of a function $f_\theta:[0,1]^d \to \mathbb{R}$.

If a function $f_\theta$ has bandlimit $B$, this means there exists coefficients $a_v$ and $b_v$ such that
\begin{equation*}
    f_\theta(x) =
\sum_{v \,\in\, \texttt{nneg}([\pm B]^d)} a_v \cos(2 \pi v^\top x)
\quad+ \sum_{v \in \texttt{nneg}([\pm B]^d) - \{\mathbf{0}\}} b_v \sin(2 \pi v^\top x),
\end{equation*}
where $[\pm B] \triangleq \{-B, \ldots, 0, \ldots B\}$ and \texttt{nneg} is defined as

\begin{minipage}{\linewidth}
\begin{minted}[numbersep=5pt,
               framesep=2mm]{python}
    def nneg(vs: list[np.ndarray[d]]) -> list[np.ndarray[d]]:
        def _include(v: np.ndarray[d]) -> bool:
            nonzero_coords = v[v != 0]
            if len(nonzero_coords) == 0:
                return True
            return nonzero_coords[0] > 0
        return [v for v in vs if _include(v)]
\end{minted}
\end{minipage}

Let $S_B$ denote the subspace of real-valued functions over $[0, 1]^d$ that is spanned by the orthonormal basis
\begin{align*}
\Phi_B \triangleq\,
& \{\cos(2\pi v^\top x) \, | \, v \in \texttt{nneg}([\pm B]^d)\} \\
\sqcup\, &\{\sin(2\pi v^\top x) \, | \, v \in \texttt{nneg}([\pm B]^d) - \{\mathbf{0}\}\}.
\end{align*}

Another way of saying $f_\theta$ has bandlimit $B$ is to say that
\begin{equation*}
\min_{g \,\in\, S_B} \;\|f_\theta - g\|^2 = 0,
\end{equation*}
where the norm $\|\cdot\|$ is defined via $\displaystyle \|f\|^2 \triangleq \int_{x \in [0,1]^d} f(x)^2 \, dx$.

\textbf{Main idea:}
In order to regularize $f_\theta$ towards having bandlimit $B$, we will add an approximate version of $\min_{g \,\in\, S_B} \;\|f_\theta - g\|^2$ to our loss. This approximation will be done via Monte-Carlo integration. It consists of the following steps:
\begin{enumerate}
    \item
    Draw uniform i.i.d. samples $x_1, \ldots, x_m$ from $[0, 1]^d$.

    \item
    Compute the vector
    \begin{equation*}
        \mathbf{y} = \begin{bmatrix}
        f_\theta(x_1) \\
        \vdots \\
        f_\theta(x_m)
        \end{bmatrix}.
    \end{equation*}

    \item
    For each function $\phi$ in $\Phi_B$, compute the vector
    \begin{equation*}
        \mathbf{v}_\phi = \begin{bmatrix}
        \phi(x_1) \\
        \vdots \\
        \phi(x_m) \\
        \end{bmatrix}.
    \end{equation*}
    Then combine all the $\mathbf{v_\phi}$ into a matrix $\mathbf{V}_B$ in the following manner:
    \begin{equation*}
        \mathbf{V}_B = \begin{bmatrix}
        | & & | \\
        \mathbf{v}_{\phi_1}
        & \cdots
        & \mathbf{v}_{\phi_{|\Phi_B|}} \\
        | & & |
        \end{bmatrix}.
    \end{equation*}
    The matrix $\mathbf{V}_B$ has $m$ rows and $|\Phi_B|$ columns.

    \item
    Finally compute 
    \begin{equation*}
        \frac{1}{m} \cdot \min_{\mathbf{z}} \|\mathbf{V}_B\, \mathbf{z} - \mathbf{y} \|_2^2 \approx \min_{g \,\in\, S_B} \;\|f_\theta - g\|^2.
    \end{equation*}
    The left hand side is equivalent to the quantity
    \begin{equation*}
        \frac{1}{m} \cdot \|\mathbf{V}_B\mathbf{V}_B^+ \,\mathbf{y} - \mathbf{y} \|_2^2,
    \end{equation*}
    where $\mathbf{V}^+_B$ is the Moore-Penrose psuedoinverse of $\mathbf{V}_B$.
\end{enumerate}

The regularized optimization in \Cref{sec:nn-harmonic-cliff} was performed on the minibatch loss
\begin{equation*}
    \frac{1}{\texttt{batch\_size}} \sum_{i = 1}^{\texttt{batch\_size}} (f_\theta(x_i) - h(x_i))^2 + \frac{\lambda}{m} \cdot \|\mathbf{V}_B\mathbf{V}_B^+ \,\mathbf{y} - \mathbf{y} \|_2^2.
\end{equation*}
with $\texttt{batch\_size} = 256$, $B = 2$, $m = 2 \times 10^4$, and $\lambda = 1$.

%% file: tex/gaussian-toy-model.tex
\section{Binary Gaussian classification toy model}
\label{sec:gtm}
In this section, we analyze a simple classification problem which demonstrates cliff-learning.

We assume the underlying data distribution
for $\mathsf{x} \in \mathbb{R}^d$ and $\mathsf{y} \in \{\pm 1\}$
is
\begin{equation*}
    \mathsf{y} \sim \text{Unif}(\{\pm 1\}), \quad
    \mathsf{x} \sim \mathcal{N}(\mathsf{y} \cdot s \cdot e_1, I_d),
\end{equation*}
where $s \geq 0$ is a hyperparameter representing the signal-to-noise ratio.

From this data distribution, we draw $n$ i.i.d.
training datapoints $\{(\mathsf{x}_i, \mathsf{y}_i)\}_{i=1}^n$.
We will use this data to learn a linear classifier to predict
$\mathsf{y}$ from $\mathsf{x}$.
Our linear classifier will take the form
\begin{equation*}
f_{w}(x) = \text{sign}(w^\top x).
\end{equation*}

We will learn our linear classifier using the following learning algorithm:
\begin{equation}
    \label{eqn:gtm-w-def}
    \hat{\mathsf{w}} = \frac{1}{n} \sum_{i=1}^n \mathsf{y}_i \mathsf{x}_i.
\end{equation}

Note that a linear change of coordinates
does not effect the performance of our learning algorithm.
Thus the results of this section apply to the more general case when
$\mathsf{x} \sim \mathcal{N}(\mathsf{y} \cdot \mu, \Sigma)$
for an arbitrary $\mu \in \mathbb{R}^d$ and positive definite $\Sigma$.

The toy model here is approximately a simplified version of the model in \citet{sorscher2022neural} -- we model $\mathsf{x}$ as Gaussian instead of uniform over an ellipsoid. This makes a difference in low dimensions, but in high dimensions the two distributions are extremely close.

\subsection{Analytic error asymptotics}
For a fixed $w$, the test error of $f_w$ can be written as
\begin{align}
\text{Err}(f_w)
&= \mathbb{P}(f_w(\mathsf{x}) \neq \mathsf{y}) \notag \\
&= \mathbb{P}(\text{sign}(w^\top \mathsf{x}) \neq \mathsf{y}) \notag \\
&= \mathbb{P}(w^\top \mathsf{x} \cdot \mathsf{y} < 0) \notag \\
&= \mathbb{P}(w^\top \mathcal{N}(s \cdot e_1, I_d) < 0) \notag \\
&= \mathbb{P}(\mathcal{N}(s \cdot w_1, \,w^\top w) < 0) \notag \\
&= \Phi\left(- \frac{s\cdot w_1}{\|w\|_2}\right), \label{eqn:gtm-err-def}
\end{align}
where $\Phi(x) = \int_{-\infty}^x \frac{1}{\sqrt{2\pi}} e^{-t^2}\, dt$ is the cumulative distribution function of the standard normal distribution.

Plugging in the learning algorithm for $\hat{\mathsf{w}}$ in \Cref{eqn:gtm-w-def}
into \Cref{eqn:gtm-err-def},
we have that
\begin{equation}
    \label{eqn:gtm-arg-def}
    \frac{s \cdot \hat{\mathsf{w}}_1}{\|\hat{\mathsf{w}}\|_2}
    \overset{d}{=}
    \frac{
        \displaystyle s^2 + \frac{s \epsilon}{\sqrt{n}}
    } {
        \displaystyle \sqrt{
            s^2
            + \frac{2s\epsilon}{\sqrt{n}}
            + \frac{\epsilon^2 + \chi^2_{d - 1}}{n}
        }
    },
\end{equation}
where $\epsilon \sim \mathcal{N}(0, 1)$ is a standard normal random variable,
$\chi^2_{d - 1}$ is a Chi-squared random variable with $d - 1$ degrees of freedom,
and $\overset{d}{=}$ denotes equality in distribution.

From \Cref{eqn:gtm-arg-def}, we get that for large $n$,
\begin{align}
    \text{Err}(f_{\hat{\mathsf{w}}})
    &\overset{d}{=}
    \Phi(-s) + \frac{\Phi'(-s)}{2s} \cdot \frac{\chi^2_{d - 1}}{n}
    + O\left(n^{-3/2}\right) \notag \\
    &\overset{d}{=}
    \Phi(-s) + \frac{e^{-s^2 / 2}}{\sqrt{8\pi} s} \cdot \frac{\chi^2_{d - 1}}{n}
    + O\left(n^{-3/2}\right),
    \label{eqn:gtm-approx}
\end{align}
which means with high probability,
$\text{Err}(f_{\hat{\mathsf{w}}})$
decays as a power law in the number of training samples $n$,
with an irreducible error of $\Phi(-s)$.

\subsection{Empirical data-scaling laws}

In \Cref{fig:gtm-main},
we show the data-scaling of $\text{Err}(f_{\hat{\mathsf{w}}})$
for various values of $d$ and $s$.
Emprically, we find that for large $d$,
the scaling curve is extremely well approximated as
\begin{equation}
    \text{Err}(f_{\hat{\mathsf{w}}}) \approx \Phi\left(
    - \frac{s}{\displaystyle \sqrt{1 + \frac{d}{ns^2}}}
    \right).
    \label{eqn:gtm-err-approx-best}
\end{equation}
When $d$ is small,
$\text{Err}(f_{\hat{\mathsf{w}}})$ has much higher variance,
but \Cref{eqn:gtm-err-approx-best} is nonetheless
a good approximation of the the median value of $\text{Err}(f_{\hat{\mathsf{w}}})$.

\begin{figure}
    \centering
    \input{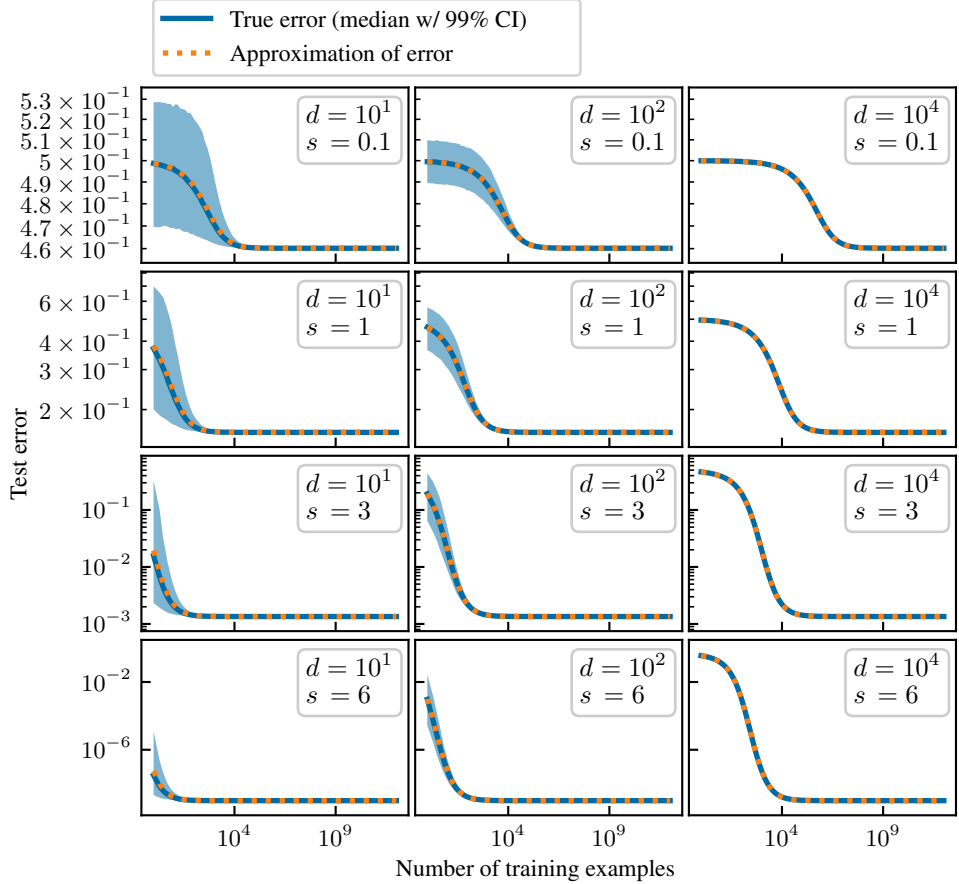}
    \caption{Data scaling laws for the Gaussian toy model defined in \Cref{sec:gtm}, showcasing a characteristic cliff shape. The blue lines denotes the median error of $f_{\hat{\mathsf{w}}}$ over 10000 trials for various values of $d$ and $s$, and shaded blue regions denote the 99\% confidence intervals over those same trials. The orange lines are computed via \Cref{eqn:gtm-err-approx-best}. Increasing $s$ has the effect of decreasing the irreducible error, and increasing $d$ reduces variance and shifts the entire curve to the right.}
    \label{fig:gtm-main}
\end{figure}